%% file: main.tex
\documentclass{article}
\usepackage{kr}

\usepackage{times}
\usepackage{soul}
\usepackage{url}
\usepackage[hidelinks]{hyperref}
\usepackage[utf8]{inputenc}
\usepackage[small]{caption}
\usepackage{graphicx}
\usepackage{amsmath}
\usepackage{amsthm}
\usepackage{booktabs}
\usepackage{algorithm}
\usepackage{algorithmic}
\newtheorem{example}{Example}
\newtheorem{theorem}{Theorem}
\newtheorem{lemma}{Lemma}

\newtheorem{proposition}{Proposition}

\title{The Polynomial Counting Capabilities \\ of Message Passing Neural Networks}

\author{%
	Marco Sälzer$^1$\and
	Pascal Bergsträßer$^1$\and
	Anthony W.\ Lin$^{1,2}$ \\
	\affiliations
	$^1$RPTU University Kaiserslautern-Landau\\
	$^2$Max Planck Institute for Software Systems (MPI-SWS)\\
	\emails
	\{marco.saelzer, pbergstr\}@rptu.de,
	awlin@mpi-sws.org
}

\usepackage{amsfonts}
\usepackage{amssymb}
\usepackage{todonotes}
\usepackage{bbm}

\usetikzlibrary{matrix,positioning}
\usetikzlibrary{arrows.meta}
\usetikzlibrary{positioning}
\input{macros}
\begin{document}

\maketitle

\begin{abstract}
    The counting power of Message Passing Neural Networks (MPNN) has been the
    subject of many recent papers, showing that they can express logic that involves 
    counting up to a threshold or more generally satisfy a linear arithmetic 
    constraint. 
    In this paper, we study the counting capabilities of MPNN beyond linear
    arithmetic, 
	primarily utilising local and global mean aggregations. In particular,
    our goal is to tease out conditions required to express extensions of
    graded modal logic with polynomial counting constraints.
    We show that global polynomial counting constraints in node-labelled 
    graphs can be checked using mean MPNN under mild assumptions. 
    Checking local constraints is also possible, 
    if we consider formulas with no nested modalities and additionally
    either (i) permit sum/max aggregations, or (ii) only restrict to
    regular graphs. We also show how formulas with nested modalities can be captured by mean
    MPNN over graphs with tree-like structures and similar assumptions.
\end{abstract}

\section{Introduction}
\input{sections/intro.tex}

\section{Preliminaries}
\label{sec:prelims}
\input{sections/prelims.tex}

\section{Shallow Peano Modal Logic Fragments}
\label{sec:polyconstraints}
\input{sections/shallow_polyconstraints.tex}

\section{Nested Peano Modal Logic Fragments}
\label{sec:nestedpolynonstraints}
\input{sections/nested_polyconstraints}

\section{Summary and Outlook}
\label{sec:outlook}
\input{sections/outlook.tex}

\paragraph{Acknowledgment.} 
We thank anonymous reviewers for their helpful feedback.
The authors are partially supported by Deutsche Forschungsgemeinschaft (grant 
number \href{https://gepris.dfg.de/gepris/projekt/522843867?language=en}{522843867}) and European Union\footnote{Views and opinions expressed are however those
	of the author(s) only and do not necessarily reflect those of the European
	Union or the European Research Council Executive Agency. Neither the
	European Union nor the granting authority can be held responsible for them.}
\includegraphics[width=0.75cm]{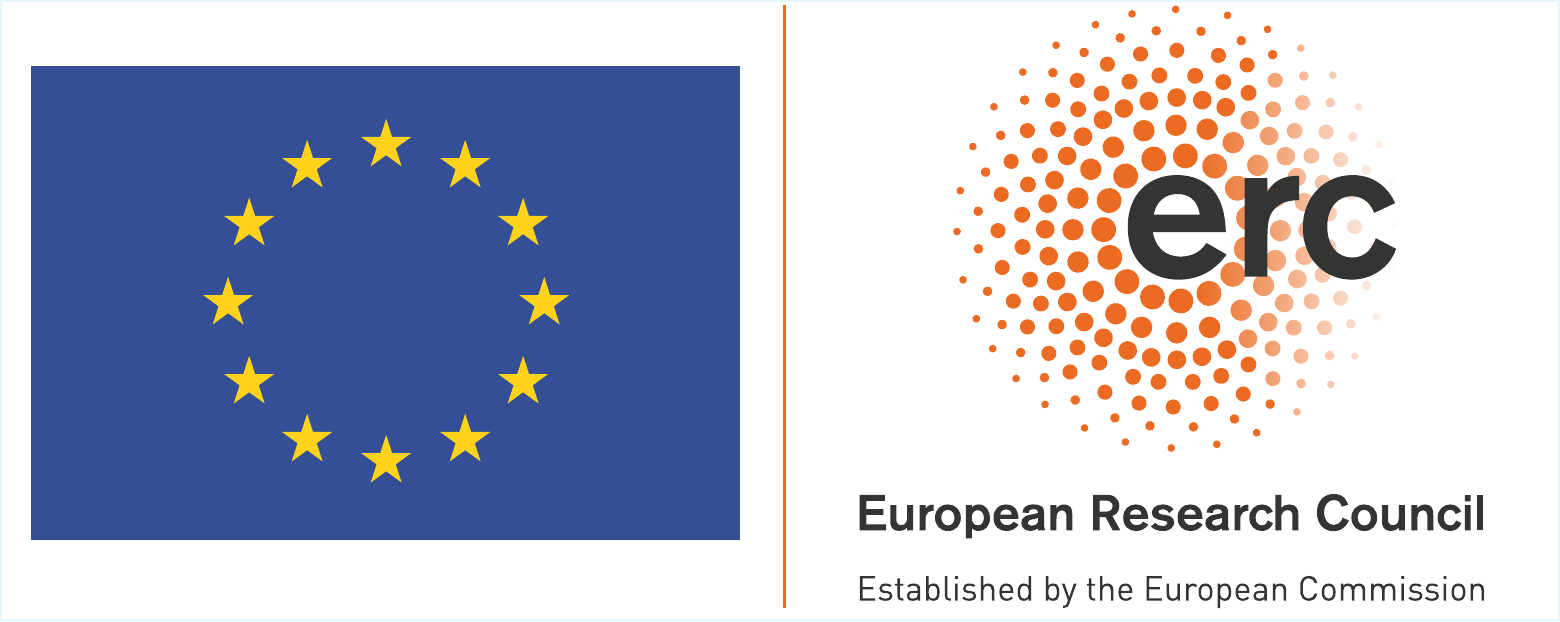}
(ERC, LASD, \href{https://doi.org/10.3030/101089343}{101089343}).

\section*{AI Declaration}

The authors have not employed any Generative AI tools.

\appendix

\section{Omitted Technicalities}
\label{app:proofs}
\input{sections/proofs.tex}

% \section{MPM2 subsumes Ksharp and MP2}
% \label{app:ksharpmp2vsmpm}
% \input{sections/app_ksharpmp2vsmpm}

%% The file kr.bst is a bibliography style file for BibTeX 0.99c
\bibliographystyle{kr}
\bibliography{refs}
\end{document}

%% file: macros.tex
\newcommand*{\nats}{\mathbb{N}}
\newcommand*{\ints}{\mathbb{Z}}
\newcommand*{\rats}{\mathbb{Q}}

\newcommand*{\mset}[1]{\{\!\{#1\}\!\}}
\newcommand*{\bs}[1]{\boldsymbol{#1}}

\newcommand*{\neighin}{\mathrm{neigh}_\text{in}}
\newcommand*{\neighout}{\mathrm{neigh}_\text{out}}
\newcommand*{\neigh}{\mathrm{neigh}}

\newcommand*{\relu}{\mathit{relu}}

\newcommand*{\comb}{\mathit{comb}}
\newcommand*{\locagg}{\mathit{loc}}
\newcommand*{\globagg}{\mathit{glob}}

\newcommand*{\meanmpnn}{\mathcal{M}_\text{mean}}

\newcommand*{\md}{\mathit{md}}

\newcommand*{\PML}{\text{PML}}

\newcommand*{\sem}[1]{[\![#1]\!]}
\newcommand*{\modalities}[2]{\textit{Mod}^{#1}_{#2}}
\newcommand*{\walks}{\mathit{Trace}}

%% file: sections/intro.tex
Message Passing Neural Networks (MPNN) are a class of neural network-based models
trained for solving tasks over structured data, i.e., graphs. 
By design, MPNN effectively utilises the structure of the input data,
in stark contrast with other approaches such as using handcrafted encoding
followed by classical feedforward neural networks (FNN).
This has led them to become one of the go-to models for addressing tasks such as drug discovery
\cite{Bongini_drugdiscovery}, a broad range of knowledge graph applications \cite{Zhou_survey},
recommender systems \cite{Wu_recommender}, or natural science applications
\cite{Kipf_natural,Shlomi_2021_natural}.

Sparked by these and other promising applications of MPNN, recent years have seen a growing line
of research thoroughly investigating the expressive capabilities of MPNN using
formal frameworks such as logics. We provide
a comprehensive overview in Section~\ref{sec:related}.
Starting with comparisons to
classical logics like fragments of first-order logic with counting (up to a
threshold) done by \cite{Barceló2020The}, later
works by \cite{NunnSST24,BenediktLMT24} have explored tailored extensions of such classical logics by 
arithmetic expressions to tighten the characterisations and infer the complexity
of verification tasks related to MPNN. 
However, all of these works have one thing in common: the logics employed for comparison with
MPNN are confined to extensions by \emph{linear integer arithmetic}. For 
example, such logics could express relationship \emph{a node has the same 
number of blue-coloured neighbours and green-coloured neighbours}. 

In this paper, our goal is to study the capability of MPNN to express counting
properties \emph{beyond} linear integer arithmetic. \emph{Counting} plays an
important role especially owing to connection to \emph{graph kernels}
(e.g. see \cite{Kriege20,WL-kernel-schweitzer,Jaakkola25})
    --- such
as Weisfeiler Leman kernel --- for classifying graphs by counting the number of
occurrences of ``neighborhoods'' in the graphs.
In the graph classification 
problem, a more general polynomial counting constraint has featured (e.g. see
\cite{ZLJ24}), which is done by applying a polynomial kernel and then Support
Vector Machines (SVM). The ability of MPNN to express polynomial counting
properties would show essentially that it generalizes the aforementioned 
technique for graph classification. 
\emph{In this paper, we identify conditions 
where MPNN can express polynomial counting properties.}

\begin{figure}
    \centering 
    \input{graphics/overview.tex}
    \caption{Overview of results. The MPNN class $\mathcal{M}_{\text{mean},x}$ refers to MPNN that 
    use mean and either sum or max as aggregation functions, $\leq$ refers to the notion of recognition with some certainty
    defined in Sec.~\ref{sec:prelims}, and $\preceq$ denotes the lowerbound on the certainty of MPNN in checking modal formulas
    which is exclusive to Lem.~\ref{sec:polyconstr;lem:globhomconstr1}.}
    \label{sec:intro:fig:overview}
\end{figure}

Informally, an MPNN uses a predefined number of message-passing rounds where each node in a
graph exchanges information along two streams: with its direct neighbours
and with all nodes in the graph. Incoming information is then aggregated per node using different functions, such
as sum, mean, or max, and is combined with the current state of a node using a learnable function, usually
represented by an FNN, which results in an updated state for each node.
Prior studies on characterising the expressive capabilities of MPNN using logic 
have primarily focused on MPNN that operate over node-labelled graphs and
utilise sum to aggregate incoming information. We identify the class of MPNN using 
mean aggregation as a candidate to posess polynomial counting 
capabilities\footnote{\cite{SchönherrL25} recently investigated the expressive capabilities
of mean MPNN in relation to specific modal logics and 
monadic second-order logic (MSO), though they do not consider genuine 
polynomial properties.}. 
In particular, we demonstrate that under certain assumptions 
mean MPNN are able to recognise different fragments of Peano modal logic (PML), a logic
that allows for different forms of modalities, local as well as global ones,
and polynomial counting constraints (not only linear ones). In particular,
PML strictly extends Graded Modal Logic \cite{Barceló2020The}.
An overview of our resuls is provided in Figure~\ref{sec:intro:fig:overview}. 
In general, we focus on identifying minimal assumptions about the class of inputs to enable
mean MPNN to recognise certain fragments of PML. Interestingly, we also 
identify settings where allowing additional sum or max aggregations eases the otherwise required assumptions.

\paragraph{Conditions.} Our conditions enabling GNN to 
perform polynomial counting (see Figure~\ref{sec:intro:fig:overview}) amount to 
restricting the input graphs. These include 
(i) adding a special marked node, (ii) regularity, and (iii) restricting input
graphs to tree-like structures. We briefly discuss their practical implications
below.

The ``marked node'' assumption needs no 
extensive justification in practice. Specifically, when using an MPNN for a 
node classification task, a simple preprocessing step marking the node under 
consideration ensures this property holds without substantially altering the 
datapoint's semantics. Similarly, the same can be ensured in the training 
phrase. The benefits of adding such preprocessing steps are theoretically 
investigated in multiple studies (e.g. \cite{Trick,Surprising}); in fact,
%such as Zhang et al. "Labeling Trick: A Theory of Using Graph Neural Networks for Multi-Node Representation Learning" or Abboud et al. "The Surprising Power of Graph Neural Networks with Random Node Initialization," 
variants thereof are used in GNN models such as SEAL or ID-GNN. Incidentally, 
a similar assumption is also used for textual data (e.g. for transformers and 
RNN), in that an end-of-string symbol and/or a beginning-of-string symbol are 
often present and beneficial (e.g. see \cite{iclr26-counting,YCA24}).

Conditions (ii) and (iii) requiring regularity or tree-like structures can be 
more restrictive. However, the assumption of, for example, tree-like inputs is 
justified by applications like learning tasks on semi-structured data such as 
those in XML, HTML and JSON, which are abundant in practice. In particular, one 
can think of classification problems of such semi-structured data. %Thus, these 
%scenarios have their relevance. 
In general, counting in trees are certainly relevant and appear in various works
(e.g. \cite{Seidl,Hague}). For example, as elaborated in detail in \cite{Hague},
the node selector language in Cascading Style Sheets (CSS) --- the de facto
language for styling web documents --- supports counting the number of children
of a matched node.
%For a general notion of why counting in trees is of interest, see Seidl et al. “Counting in trees” (2008) and also Hague et al. "CSS Minification via Constraint Solving" (2019). Also, if a precondition like tree-like graphs is not met, Theorems 3 to 6 show that (mean) MPNN can capture $\text{PML}_1$ over general graphs. This fragment allows for arbitrary global polynomial counting constraints while simultaneously considering counting constraints over direct neighbourhoods. As underlined by Example 3, this is already quite powerful.

\paragraph{Organization.} The paper is structured as follows.
In Section~\ref{sec:prelims}, we define the necessary fundamentals,
such as MPNN, PML, and what it means for the former to recognise the latter.
In Section~\ref{sec:polyconstraints}, we investigate the extent to which MPNN can
recognise shallow fragments of PML, meaning that we do
not nest any modal subformulas.
Subsequently, we dedicate Section~\ref{sec:nestedpolynonstraints} to the analysis
of the capabilities of MPNN in relation to PML fragments with nesting.
Finally, in Section~\ref{sec:outlook}, we summarise and discuss the next steps.
Some formal arguments are of a rather technical nature and are used
repeatedly; thus, we have omitted them in the main part and deferred
them to Appendix~\ref{app:proofs}.

\subsection{Related Work}
\label{sec:related}

The seminal work by \cite{Barceló2020The} investigates the expressive 
capabilities of MPNN using sum aggregation exclusively via fragments of the logic
$\text{FOC}_2$, which denotes the two-variable fragment of first-order logic with counting.
Two prominent results are that all graded modal logic classifiers are captured by local
MPNN with sum aggregation and that full $\text{FOC}_2$ is captured by MPNN also using
global sum aggregation. These results are succinctly summarised and placed in context of
characterisations based on the Weisfeiler-Leman algorithm in \cite{Grohe21}, though the
latter is only loosely connected to this paper. Subsequent studies such as done by 
\cite{BenediktLMT24} and \cite{NunnSST24} look at 
extensions of classical logics by means of linear arithmetic to further characterise 
the expressive capabilities of MPNN with sum aggregation. Notably, besides results about the complexity 
of reasoning about such MPNN, these studies deliver upper bounds on their expressiveness, but only under certain assumptions such as eventually constant
activation functions. 

Going beyond plain sum MPNN, \cite{CucalaGMK23} show that
monotonic MPNN using either maximum or sum aggregation can be captured by Datalog programs.
Similarly, \cite{AhvonenHKL24} examine the expressive capabilities of recursive MPNN,
those with input-dependent depth, in relation to properties expressible in monadic second-order logic (MSO).
Notably, this study considers differences arising from assumptions about the underlying arithmetic, such as
real arithmetic versus floating-point arithmetic, showing that assuming the latter leads to generally weaker MPNN.
\cite{RosenbluthTG23} compared differences in the expressive capabilities of 
MPNN exclusively using sum, mean or max aggregation. 
A work closely related to the setting of this paper is the recent study by \cite{SchönherrL25}, which examines the logical expressiveness
of MPNN with mean aggregation. They derive several results based on modal logics,
such as ratio modal logic (RML), which allows for expressing proportional
properties of neighbourhoods. 

We remark that in research on capturing the expressive capabilities of transformer architectures,
there is also a line looking into their logical expressiveness \cite{ChiangCP23,BarceloKLP24}.
While this appears only loosely related at first glance, MPNN with powerful aggregations,
such as graph attention networks (GAT) \cite{VeličkovićCCLR18}, can be considered as a
generalisation, leading to many similarities in these types of results. Notably, the very recent study 
\cite{iclr26-counting} examines the polynomial counting capabilities of transformers
using average hard attention, employing constructions similar to ours.

%% file: graphics/overview.tex
\begin{tikzpicture}

    \draw (0,0) rectangle (8,6);

    \node at (.8,5.7) {\small \textit{all graphs}};

    \node at (4,5.7) {\small$\PML^{\top,\mathcal{H}}  \preceq \meanmpnn$};
    \node at (4,5.3) {\tiny (Lem.~\ref{sec:polyconstr;lem:globhomconstr1}, Prop.~\ref{prop:global-pml-depth})};

    \draw (.1,0) rectangle (7.9,5);
    \node at (.7,4.7) {\small \textit{marked}};

    \node at (4,4.3) {\small $\PML^{\top} \leq \meanmpnn$};
    \node at (4,3.9) {\tiny (Th.~\ref{sec:polyconstr;th:globconstr1}, Prop.~\ref{prop:global-pml-depth})};
    \node at (2.1,2.75) {\small $\PML_1 \leq \mathcal{M}_{\text{mean},x}$};
    \node at (2.1,2.35) {\tiny (Th.~\ref{sec:polyconstr;th:localconstraintssummeanmax}, Th.~\ref{sec:polyconstr;th:mixedmodmeansummax})};

    \node at (1.3,3.3) {\small \textit{strongly marked}};
    
    \node at (6.2,3.3) {\small \textit{regular strongly marked}};

    \node at (6,2.75) {\small $\PML_1 \leq \meanmpnn$};
    \node at (6,2.35) {\tiny (Th.~\ref{sec:polyconstr;th:localconstraintsmean}, Th.~\ref{sec:polyconstr;th:mixedmodmean})};

    \draw (.2,0) rectangle (7.8,3.5);

    \draw (.3,0) rectangle (7.7,2);
    \node at (.9,1.8) {\small \textit{tree-like}};
    \node at (2.2,1) {\small $\PML^E \leq \mathcal{M}_{\text{mean},x}$};
    \node at (2.1,.6) {\tiny (Th.~\ref{sec:nested;th:mpnnmeansummax_and_mixed})};

    \draw(4,3.5) -- (4,0);

    \node at (6.6,1.8) {\small \textit{regular tree-like}};
    \node at (6,1) {\small $\PML^E \leq \meanmpnn$};
    \node at (6,.6) {\tiny (Th.~\ref{sec:nested;th:meanmpnn})};
\end{tikzpicture}

%% file: sections/prelims.tex
We denote vectors using bold symbols such as $\bs{x}$, $\bs{y}$, or $\bs{z}$. We use $\bs{x}_\bot$ to 
denote the element in the last dimension of vector $\bs{x}$. We denote \emph{multisets}, 
which are sets containing duplicate elements, using the notation $\mset{\cdot}$.

A \emph{(directed) graph} $G$ is a tuple $(V,E,L)$ where $V$ is a finite set 
of vertices, $E \subseteq V\times V$, and $L \colon V \to \{0,1\}^k$ is
a labelling function. We refer to the different dimensions of $L$ as \emph{colours} and say that \emph{$G$ has $k$ colours}.
We denote the class of all directed graphs by $\mathcal{G}$.
Let $v \in V$. We define the \emph{ingoing (outgoing) neighbourhood of $v$} by 
$\neighin(v) = \{u \mid (u,v) \in E\}$ ($\neighout(v) = \{u \mid (v,u) \in E\}$), and the general \emph{neighbourhood of $v$} by $\neigh(v) = \neighin(v) \cup \neighout(v)$. 
We call $G$ \emph{ingoing (outgoing) regular} if 
for all $u,v \in V$ we have $|\neighin(v)| = |\neighin(u)|$ ($|\neighout(v)| = |\neighout(u)|$). If both are given, we refer to $G$ 
simply as \emph{regular}.
We call a pair $(G,v)$ a \emph{pointed graph} and call $v$ the \emph{focus}.
We call a node $v$ in $G$ \emph{marked} if there is dimension $i$ of the labelling $L$ such that for all $u \in V$
we have $L(u)_i = 1$ if and only if $u=v$. Let $\mathcal{G}'$ be any class of directed graphs. 
We denote by $(\mathcal{G}')^p$ for $p \in \nats$ the class of all graphs $G \in \mathcal{G}'$ that contain a node 
$v$ that is marked by colour $p$. We call $v$ \emph{strongly marked} if its marked and $v \in \neigh(v)$. 
Likewise, we denote by $(\mathcal{G}')^p_\bullet$ for $p \in \nats$ the class of all pointed graphs 
$(G,v)$ where $G \in \mathcal{G}'$ and focus $v$ is marked by colour $p$.
Let $v,v' \in V$. 
We call a sequence $u_0, \dotsc, u_{n-1}$ with $u_0 = v$, $u_{n-1} = v'$,
and $(u_i, u_{i+1}) \in E$ or $(u_{i+1}, u_i) \in E$ for all $0 \leq i < n-1$
a \emph{walk of length $n-1$ from $v$ to $v'$}.
If such a walk exists from $v$ to $v'$,
we say that $v'$ is at a walking distance $(n-1)$ from $v$.

\subsection*{Feedforward Neural Networks with ReLU}
A \emph{Feedforward Neural Network (FNN) with ReLU activations} is built from basic computational units, employing the ReLU activation
$\relu(x) = \max(0, x)$, called a \emph{ReLU neuron}. A ReLU neuron is defined as a mapping $\mathbb{Q}^m \to \mathbb{Q}$
by $v(x_1, \dots, x_m) = \relu(b + \sum_{i=1}^m w_i x_i)$, where $w_i \in \mathbb{Q}$ are the \emph{weights} and $b \in \mathbb{Q}$ is the \emph{bias}.
An \emph{FNN layer} $\ell$ consists of a tuple of ReLU neurons $(v_1, \dots, v_n)$, each having the same input dimension, realizing
a mapping $\mathbb{Q}^m \to \mathbb{Q}^n$. A complete \emph{FNN} $N$ is then structured as a sequence of such layers $(\ell_1, \dots, \ell_k)$
where the output dimension of each layer $\ell_i$ matches the input dimension of the next layer $\ell_{i+1}$.
The network $N$ calculates a function from $\mathbb{Q}^{m_1}$ to $\mathbb{Q}^{n_k}$, given by
$\mathcal{N}(x_1, \dots, x_{m_1}) = \ell_k(\ldots \ell_1(x_1, \dots, x_{m_1}) \ldots)$.

We primarily use FNN as components within message passing neural networks, as defined below.
To do so, we frequently employ the term \emph{gadget} to describe combinations of ReLU nodes that perform a specific function.
When these gadgets are integrated into larger FNN, it is always given that this integration is feasible
in the sense of forming a well-defined FNN.

\subsection*{Message Passing Neural Networks}

A \emph{message passing neural network} (MPNN) $M$ \cite{Barceló2020The,GilmerSRVD17} is a tuple $(l_1, \dotsc, l_L)$ where each $l_i$ is a so-called \emph{layer} 
given by $l_i = (\comb_i, (\locagg_\text{in})_i, (\locagg_\text{out})_i, \globagg_i)$, consisting of a \emph{combination function} $\comb_i$, two 
\emph{local aggregation} $(\locagg_x)_i$, one for ingoing and one for outgoing edges, and a \emph{global aggregation} $\globagg_i$.
Let $G=(V,E,L)$ be a graph. MPNN $M$ computes a new \emph{state} $\bs{x}^{(i)}_v$ for each node $v \in V$ by $\bs{x}^{(0)}_v = L(v)$, 
and $\bs{x}^{(i)}_v = \comb_i(\bs{x}^{(i-1)}_v, \{(\locagg_{x})_i(\mset{\bs{x}^{(i-1)}_u \mid u \in \neigh_{x}(v)})\}_{x\in\{\text{in},\text{out}\}}, 
\globagg_i(\mset{\bs{x}^{(i-1)}_u \mid u \in V}))$. We also write $M(G,v)$ for the final state $\bs{x}^{(L)}_v$ of node $v \in V$.
We generally assume that combination functions $\comb_i$ 
are realised by feedforward neural networks (FNN) using ReLU 
%($\relu(x) = \max(0,x)$) 
activations. 
Let $\mathcal{M}$ be a class of arbitrary MPNN.
In particular, we consider the class of all MPNN $M$ where all local and global aggregations are realised by 
\begin{displaymath}
    f(\mset{\bs{x}_1, \dotsc \bs{x}_n}) = \frac{1}{|\mset{\bs{x}_1, \dotsc \bs{x}_n}|} \sum_{i=1}^n \bs{x}_i, 
\end{displaymath}
We also consider MPNN that use 
sum or max aggregations, defined by
\begin{align*}
    f(\mset{\bs{x}_1, \dotsc \bs{x}_n}) &= \sum_{i=1}^n \bs{x}_i, \\ 
    f(\mset{\bs{x}_1, \dotsc \bs{x}_n}) &= \max(\bs{x}_1, \dotsc \bs{x}_n), 
\end{align*}
where $\max$ means taking the dimensionwise maximum. Let $A \subseteq \{\text{mean}, \text{sum}, \text{max}\}$. 
In the case that $n=0$, we always assume that $f(\mset{\,}) = \bs{0}$ of matching dimensionality.
We refer to the class of MPNN that use aggregations from $A$ by $\mathcal{M}_A$.

\subsection*{Peano modal logic}

We consider modal logics extended with Peano arithmetic. We begin 
by defining Peano arithmetic constraints over graphs. 
Let $X$ be a countable set of variables. A \emph{Peano formula} $\psi$ is defined by the grammar
\begin{align*}
  \psi &::= \zeta \leq c \mid \psi \land \psi \mid \neg \psi \\ 
  \zeta &::= x \mid c \cdot \zeta \mid \zeta + \zeta \mid \zeta \cdot \zeta 
\end{align*}
where $x \in X$ and $c \in \ints$. We call $\zeta$ a \emph{(Peano) term}. We use $\psi(x_1, \dotsc, x_m)$ or $\zeta(x_1, \dotsc, x_m)$ 
to denote that $x_1$ to $x_m$ are the variables occurring in $\psi$ or $\zeta$, respectively.
The \emph{semantics of $\psi(x_1, \dotsc, x_m)$}, denoted by $\sem{\psi(x_1, \dotsc, x_m)}$, are all tuples $(n_1, \dotsc, n_m) \in \nats^m$ that 
satisfy $\psi$, written $n_1, \dotsc, n_m \models \psi$, in the obvious sense. We generally assume that Peano terms are given in the normalised form 
$\sum_{i=1}^K a_i \cdot \prod_{j=1}^{k_i} x_{i_j} \leq b$, refer to $a_i \cdot \prod_{j=1}^{k_i} x_{i_j}$ as a \emph{monomial} and to $k_i$ as the \emph{degree} of this monomial.
Then, we call $k$ \emph{the degree of $\psi$}, denoted by $\mathit{deg}(\psi)$, if it is the maximum of the degrees of all monomials occurring in $\psi$.

Let $P$ be a countable set of propositions. A \emph{Peano modal logic (PML)} formula $\varphi$ is defined by the grammar 
\begin{align*}
  \pi &::= \mathit{id} \mid E_\text{in} \mid E_\text{out} \mid \top \\ %\mid \pi \cup \pi \mid \pi \cap \pi \mid \neg \pi \\
  \varphi &::= p \mid \neg \varphi \mid \varphi \land \varphi \mid \langle \pi_1, \dotsc, \pi_m \rangle_{\psi(x_1, \dotsc, x_m)}(\varphi_1, \dotsc, \varphi_m)  
\end{align*}
where $p \in P$ and $\psi(x_1, \dotsc, x_m)$ is a Peano formula as defined above. We call $\pi$ a \emph{modality}, and refer to 
$\langle \pi_1, \dotsc, \pi_m \rangle_{\psi(x_1, \dotsc, x_m)}(\varphi_1, \dotsc, \varphi_m) $ as \emph{modal formulas}.

We define the \emph{modal depth} of a PML formula 
inductively by $\mathit{md}(p) = 0$, $\mathit{md}(\neg\varphi)=\mathit{md}(\varphi)$, 
$\mathit{md}(\varphi_1\land\varphi_2)=\max(\mathit{md}(\varphi_1),\mathit{md}(\varphi_2))$, and 
$\mathit{md}(\langle \pi_1, \dotsc, \pi_m \rangle_{\psi}(\varphi_1, \dotsc, \varphi_m)) = 1 + \max(\mathit{md}(\varphi_1), \dotsc, \mathit{md}(\varphi_m))$.
Likewise, we define the \emph{set of subformulas of $\varphi$}, denoted by $\mathit{sub}(\varphi)$, inductively by
$\mathit{sub}(p) = \{p\}$, $\mathit{sub}(\neg\varphi)=\mathit{sub}(\varphi) \cup \{\neg\varphi\}$, 
$\mathit{sub}(\varphi_1\land\varphi_2)=\mathit{sub}(\varphi_1) \cup \mathit{sub}(\varphi_2) \cup \{\varphi_1\land\varphi_2\}$, and 
$\mathit{sub}(\langle \pi_1, \dotsc, \pi_m \rangle_{\psi}(\varphi_1, \dotsc, \varphi_m)) = \{\langle \pi_1, \dotsc, \pi_m \rangle_{\psi}(\varphi_1, \dotsc, \varphi_m)\} \cup 
\bigcup_{i=1}^m \mathit{sub}(\varphi_i)$. 
We denote by $\modalities{k}{\varphi}$ the set of all modalities occurring in modal subformulas of 
$\varphi$ with modal depth $k$. 
We extend the definition of degree from Peano formulas to PML formulas by defining 
the \emph{degree of $\varphi$}, denoted by $\mathit{deg}(\varphi)$ as the maximum of all degrees of Peano formulas $\psi$ in $\varphi$.
We define $\walks^k_\varphi \subseteq \{E_{\text{in}}, E_{\text{out}}\}^k$
for all $1 \le k \le \mathit{md}(\varphi)$ as the set of all sequences
$E_0, \dotsc, E_{k-1}$ for which there exists a sequence of subformulas
$\varphi_0, \dotsc, \varphi_{k}$ of $\varphi$ such that 
for all $i < k$ there is $1 \leq j \leq m_i$ such that 
$\varphi_i = \langle \pi_1, \dotsc, \pi_{m_i} \rangle_{\psi} (\psi_1, \dotsc, \psi_{m_i})$
with $\varphi_{i+1} \in \mathit{sub}(\psi_j)$ and $E_i = \pi_j$ and $\md(\varphi_i) = \md(\varphi)-i$.

The \emph{semantics of a PML formula $\varphi$,} denoted by $\sem{\varphi}$, is the set 
of all pointed graphs $(G,v)$ with $G=(V,E,L)$ that satisfy $\varphi$, written $G,v \models \varphi$, which is defined as follows:
\begin{align*}
  &G,v \models p_i && \text{iff} && L(v)_i = 1, \\ 
  &G,v \models \neg \varphi && \text{iff} && G,v \not\models \varphi, \\ 
  &G,v \models \varphi_1 \land \varphi_2 && \text{iff} && G,v \models \varphi_1 \text{ and }  G,v \models \varphi_2 , \text{ and}
\end{align*}
$G,v \models \langle \pi_1, \dotsc, \pi_m \rangle_{\psi(x_1, \dotsc, x_m)}(\varphi_1, \dotsc, \varphi_m)$ iff 
\begin{itemize}
  \item for all $i \in \{1,\dotsc, m\}$ there are exactly $n_i \in \nats$ nodes $u \in V$ such that $u \in \sem{\pi_i}^G_v$ and $G, u \models \varphi_i$, and 
  \item $n_1, \dotsc, n_m \models \psi(x_1, \dotsc, x_m)$,
\end{itemize}
where the \emph{semantics of modality $\pi$ given $(G,v)$}, written $\sem{\pi}^G_v$, are defined by 
\begin{align*}
  &\sem{\mathit{id}}^G_v && = && \{v\}, \\ 
  &\sem{E_\text{in}}^G_v && = && \{u \mid (u,v) \in  E\}, \\
  &\sem{E_\text{out}}^G_v && = && \{u \mid (v,u) \in  E\}, \\ 
  &\sem{\top}^G_v && = && V.
\end{align*}

We consider certain fragments of PML. 
We define $\PML_i$ for $i \in \nats$ as the fragment of all PML formulas of modal depth 
at most $i$. Let $\mathcal{L}$ be some fragment of PML.
We define $\mathcal{L}^\mathcal{H}$ as the fragment of 
all formulas $\varphi \in \mathcal{L}$ where for all modal subformulas the corresponding Peano formula is of the form
$\psi=\sum_{i=1}^K a_i \cdot \prod_{j=1}^{k} x_{i_j} \leq 0$.
We define $\mathcal{L}^\top \subseteq \mathcal{L}$ as the fragment of all formulas that only use $\top$ as modality. 
Likewise, we define $\mathcal{L}^E \subseteq \mathcal{L}$ as the fragment of all formulas that only use $E_\text{in}$ or 
$E_\text{out}$ as modality.
We also use combinations 
such as $\mathcal{L}^{\top,\mathcal{H}}$ in the obvious sense.

\subsection*{Connection between MPNN and PML}
Let $\mathcal{M}$ be a class of MPNN, $\mathcal{L}$ a fragment of PML,
$\mathcal{G}_\bullet$ a class of pointed graphs, and $c: \nats \to (0;1]$.
We say that $M \in \mathcal{M}$ \emph{recognises} $\varphi$ \emph{over} $\mathcal{G}_\bullet$ \emph{with certainty} $c(|V|)$
if for every $(G=(V,E,L),v) \in \mathcal{G}_\bullet$, it holds that $M(G,v)_\bot = c(|V|)$ when $(G,v) \models \varphi$,
and $M(G,v)_\bot = 0$ when $(G,v) \not\models \varphi$. If for each $\varphi \in \mathcal{L}$,
there is an $M \in \mathcal{M}$ that recognises it in the above sense for some $c$ dependent of $\varphi$, 
we also denote this by $\mathcal{L} \leq \mathcal{M}$ and assume that its clear from context 
to which pointed graph class it relates.
Informally, the notion that $M$ recognises some $\varphi$ is inspired by binary classification tasks,
namely that $M$ aligns with $\varphi$ up to an input-dependent threshold determined by $c$.

%% file: sections/shallow_polyconstraints.tex
In this section, we focus on analysing the expressive capabilities of MPNN in relation to polynomial counting constraints 
with shallow modal depth. Formally, we frame this by focusing on fragments of $\PML_1$, which is the fragment of 
PML where all formulas are of modal depth at most one. Conceptually, we separate our analysis based on the 
modalities used.

\subsection*{Global Modalities Only}
As a first step, we consider the fragment $\PML_1^{\top,\mathcal{H}}$, which includes only modal subformulas exlusively using
\emph{homogeneous} Peano terms, which are those where all monomials have the same degree and no constant term.
\footnote{We remark that Lem.~\ref{sec:polyconstr;lem:globhomconstr1} is not a recognition 
result in the formal sense defined in Section~\ref{sec:prelims}, but provides a lower bound on the ``certainty'' with which an MPNN recognises 
that a modal formula $\varphi \in \PML_1^{\top,\mathcal{H}}$ is not satisfied by a given pointed graph.}
\begin{lemma}
    \label{sec:polyconstr;lem:globhomconstr1}
    Let $\varphi = \langle \pi_1, \dotsc, \pi_m \rangle_{\psi}(\varphi_1, \dotsc, \varphi_m) \in \PML^{\top,\mathcal{H}}_1$. 
    There is $M_\varphi \in \meanmpnn$ with $\mathit{deg}(\psi)+1$ layers 
    such that for all graphs $G = (V,E,L)$ and $v \in V$ we have
    \begin{itemize}
      \item $(\bs{x}^{(\mathit{deg}(\psi)+1)}_v)_\bot \geq \frac{1}{|V|^{\mathit{deg}(\varphi)}}$ if $(G,v) \not\models \varphi$, and 
      \item $(\bs{x}^{(\mathit{deg}(\psi)+1)}_v)_\bot = 0$ otherwise.
    \end{itemize}
\end{lemma}
\begin{proof}
  Let $\chi_1, \dotsc, \chi_n$ be an enumeration of the subformulas of $\varphi$ such that 
  \begin{enumerate}
    \item if $\chi_j$ is a subformula of $\chi_i$ then $j \leq i$, and  
    \item $\chi_n = \varphi$.
  \end{enumerate}
  Assume that $\mathit{deg}(\varphi) = k$.
  We construct $M_\varphi \in \meanmpnn$ as follows. In the first layer $l_1$, we evaluate all subformulas $\chi_i$ for $i \leq n-1$,
  such that for all graphs $G$ and nodes $v$, we have $(\bs{x}^{(1)}_v)_i = 1$
  if $(G,v) \models \chi_i$ and $(\bs{x}^{(1)}_v)_i = 0$ if not. Note that these subformulas are exclusively Boolean and, thus, we 
  refer to Lemma~\ref{app:omitted;lem:boolenformulas} for the construction of $l_1$.

  Now, consider the last subformula $\chi_n = \varphi = \langle \pi_1, \dots, \pi_m \rangle_{\psi(x_1, \dots, x_m)}(\varphi_1, \dots, \varphi_m)$. By assumption, we have
  \begin{itemize}
    \item $\psi(x_1, \dots, x_m) = \sum_{i=1}^K a_i \prod_{j=1}^k x_{i_j} \leq 0$, where $i_j \in \{1,\dots, m\}$,
    \item $\pi_i = \top$ for all $i \leq m$.
  \end{itemize}
  This indicates that the Peano formula $\psi(x_1, \dots, x_m)$ needs to be evaluated over values
  $x_{i_j}$ with $i_j\in \{1, \dotsc, m\}$, which represent the total number of nodes in $G$ satisfying the purely Boolean formulas $\varphi_{i_j}$.
  First, we describe how we construct subsequent layers of $M_\varphi$ 
  to evaluate a single monomial $\prod_{j=1}^k x_{i_j}$ of $\psi$. 
  We add $k$ layers $l_{2}, \dotsc, l_{2+k-1}$ to $M_\varphi$ such that
  $\bs{x}^{(2+(h-1))}_v$ contains $\frac{1}{|V|^h} \prod_{j=1}^h x_{i_j}$
  and $\bs{x}^{(2+({h'}-1))}_v$ contains $x_{i_{h'+1}} \cdot \frac{1}{|V|^{h'}} \prod_{j=1}^{h'} x_{i_j}$
  for all $1 \leq h \leq k$ and $1 \leq h' < k$. 
  In the case $h=1$, we use global mean aggregation of $M_\varphi$ in layer $l_2$ over the dimension of the states $\bs{x}^{(1)}_u$ with $u \in V$
  corresponding to $\varphi_{i_1}$. This yields $\frac{1}{|V|} x_{i_1}$ in $\bs{x}^{(2)}_v$ for all $v \in V$,
  where $x_{i_1}$ represents the number of $u \in V$ such that $G, u \models \varphi_{i_1}$. Then, if $k > 1$, we add the gadget
  $\relu((\frac{1}{|V|} x_{i_1})-(1-(\bs{x}^{(1)}_v)_{i_2}))$ to perform the multiplication $(\bs{x}^{(1)}_v)_{i_2} \cdot \frac{1}{|V|} x_{i_1}$ 
  for all $v \in V$.
  Note that this gadget works as intended 
  due to the fact that $\frac{1}{|V|}x_{i_1}\in [0,1]$ and $(\bs{x}^{(1)}_v)_{i_2} \in \{0,1\}$. For $h > 1$ it works 
  analogously, but we aggregate over the previously computed multiplication value.

  Given that we have evaluated the monomial $\prod_{j=1}^k x_{i_j}$ for all $i \leq K$ this way, 
  we can then compute whether $G,v \not\models \varphi$ as stated by the theorem, in layer $l_{2+k-1}$ as follows.
  Since $\psi$ is homogeneous, the denominator in
  $\frac{1}{|V|^k} \prod_{j=1}^k x_{i_j}$ is equal for all $i\leq K$.
  Therefore, we exploit the fact that $\sum_{i=1}^K a_i \prod_{j=1}^k x_{i_j} \leq 0$ if and only if
  $\frac{1}{|V|^k}\sum_{i=1}^K a_i \prod_{j=1}^k x_{i_j} \leq 0$. Moreover, we note that
  if the inequality is not satisfied, then
  $\frac{1}{|V|^{k}}\sum_{i=1}^K a_i \prod_{j=1}^k x_{i_j} \geq \frac{1}{|V|^{k}}$.
  This arises from the fact that $x_{i_j} \in \nats$ for all $i_j$. Thus, we add to
  $\comb_{2+k-1}$ a gadget computing $\relu(\sum_{i=0}^K a_i \cdot \frac{1}{|V|^k} \prod_{j=1}^k x_{i_j})$ 
  whose output we make the last dimension of state $\bs{x}^{(2+k-1)}_v$.
\end{proof}

While a restriction, homogeneous polynomial constraints are quite powerful already. Informally, they allow 
to compare (polynomial) proportions of properties in graphs. 
\begin{example}
  Consider the formula $\varphi = \langle \top, \top, \top \rangle_\psi(p_r,p_g,p_b)$ where 
  $\psi(x_r,x_g,x_b) = x_r^3 - x_g^2x_b \leq 0$. Clearly, we have $\varphi \in \PML^{\top, \mathcal{H}}_1$.
  Informally, interpret proposition $p_r$ as colour \emph{red}, $p_g$ as colour \emph{green}, and $p_b$
  as colour \emph{blue}. Then, the formula $\varphi$ is satisfied by exactly those graphs $G$ where the cube of the number
  of red nodes is at most as great as the squared number of green nodes multiplied by the number of blue nodes.
\end{example}

The obvious next question is whether mean MPNN can recognise all formulas in $\PML_1^\top$. For this purpose, the arguments used in
Lemma~\ref{sec:polyconstr;lem:globhomconstr1} prove to be insufficient since the denominators $|V|^i$
introduced while evaluating monomials are no longer guaranteed to be identical, as the degree $i$ of
monomials in Peano terms may vary. Moreover, potential constant terms in Peano expressions also need to be considered.
It is, in fact, straightforward to demonstrate that for arbitrary inputs, mean MPNN are unable to express all such constraints without 
further restrictions.
\newcommand{\ptgraphGOne}{%
\begin{tikzpicture}[baseline=-0.6ex,
  every node/.style={circle,draw,inner sep=0pt,minimum size=3mm},
  pointed/.style={line width=0.9pt},
]
\node[pointed, fill=blue!50] (v1) {};
\node[right=3.5mm of v1, fill=white] (v2) {};
\end{tikzpicture}%
}

\newcommand{\ptgraphGTwo}{%
\begin{tikzpicture}[baseline=-0.6ex,
  every node/.style={circle,draw,inner sep=0pt,minimum size=3mm},
  pointed/.style={line width=0.9pt},
]
\node[pointed, fill=blue!50] (a) {};
\node[right=2.8mm of a, fill=blue!50] (b) {};
\node[right=2.8mm of b, fill=white] (c) {};
\node[right=2.8mm of c, fill=white] (d) {};
\end{tikzpicture}%
}
\begin{lemma}
  \label{sec:polyconstr;lem:inexpress}
  Let $\varphi = \langle \top \rangle_{x_1 \geq 2}(p) \in \PML_1^\top$. There is no $M \in \meanmpnn$ such that 
  for all pointed graphs $(G,v)$ we have $M$ accepts $(G,v)$ if and only if $(G,v) \models \varphi$.
\end{lemma}
\begin{proof}
  Assume the contrary, namely that there exists an $M \in \meanmpnn$ that
  recognises $\varphi$ exactly. W.l.o.g.\  assume that $p$ corresponds
  to the label dimension 0. Consider the two pointed graphs $(G_1=(V_1,E_1,L_1),v_1)$ with
  $V_1 = \{v_1,v_2\}$, $E_1 = \emptyset$, $L_1(v_1) = (1)$, $L_1(v_2) = (0)$, and
  $(G_2=(V_2,E_2,L_2),v_1')$ with $V_2 = \{v_1',v_2',v_3',v_4'\}$, $E_2 = \emptyset$,
  $L_2(v_i') = (1)$ if $i \leq 2$, and $L_2(v_i') = (0)$ if $i \geq 3$. Visually, we have 
  \begin{align*}
    (G_1,v_1)=\ptgraphGOne &&\text{and}&& (G_2,v_1')=\ptgraphGTwo,
  \end{align*}
  where the blue filling denotes 
  label $(1)$, white filling denotes label $(0)$, and the focus is thickly bordered.
  It is easy to see that $(G_1,v_1) \not\models \varphi$ and $(G_2, v_1') \models \varphi$.
  However, it is also straightforward that for each pair of nodes 
  $(u,u') \in V_1 \times V_2$ with $L_1(u) = L_2(u')$ that all applications of global mean aggregations lead to 
  the same values. Given that, we have that $\bs{x}^{(k)}_{v_1} = \bs{x}^{(k)}_{v_1'}$, 
  where $k$ is the last layer of $M$, this contradicts the  fact that $M$ does not accept $(G_1, v_1)$ but $(G_2,v_1')$.
\end{proof}

We identify a simple assumption about input graphs, enabling mean MPNN 
to recognise general polynomial constraints: there is a marked node, meaning one node that has a colour not shared by any other node.

\begin{theorem}
    \label{sec:polyconstr;th:globconstr1}
    Let $\varphi \in \PML_1^\top$.
    There is $M_\varphi \in \meanmpnn$ with $\mathcal{O}(\mathit{deg}(\varphi))$ layers 
    that recognises $\varphi$ over $(\mathcal{G})^p_\bullet$ with certainty $\frac{1}{|V|^{\mathcal{O}(\mathit{deg}(\varphi))}}$.
\end{theorem}
\begin{proof}
  The proof works similar to Lemma~\ref{sec:polyconstr;lem:globhomconstr1}, namely we give a 
  construction for $M_\varphi$. 
  Again, we enumerate the subformulas of $\varphi$ such that if $\chi_j$ is a subformula
  of $\chi_i$, then $j \leq i$, with $\chi_n = \varphi$. Additionally, we ensure that there is some $h \leq n$ such 
  that for all $\chi_j$ we have $\mathit{md}(\chi_j) = 0$ if $j \leq h$ and $\mathit{md}(\chi_j) = 1$ if 
  $j > h$. In other words, the enumeration starts with all formulas that do not contain a modal subformula.  
  Let $G=(V,E,L) \in \mathcal{G}^p$ be a graph with marked node $u \in V$.
  We construct $M_\varphi$ as follows. Layer $l_1$ is constructed as in Lemma~\ref{app:omitted;lem:boolenformulas}, 
  evaluating all formulas $\chi_i$ with $i \leq h$.

  Consider $\chi_i = \langle \pi_1, \dotsc, \pi_m \rangle_{\psi(x_1, \dotsc, x_m)}(\varphi_1, \dotsc, \varphi_m)$.
  By assumption, we know that $\mathit{md}(\chi_i) = 1$.
  In contrast to the setting of Lemma~\ref{sec:polyconstr;lem:globhomconstr1}, there are three differences regarding 
  the Peano formula $\psi$:
  \begin{enumerate}
    \item monomials of a term $\sum_{i=1}^K a_i \cdot \prod_{j=1}^{k_i} x_{i_j} \leq b$ are of the form $\prod_{j=1}^{k_i} x_{i_j}$, indicating they may have varying degrees $k_i$,
    \item terms may include a constant part $b \neq 0$, and 
    \item $\psi$ can be a genuine Peano formula, meaning its a Boolean combination of Peano terms.
  \end{enumerate} 

  We continue by describing how to evaluate monomials and constant parts.
  Let $\mathit{deg}(\psi) = \hat{k} = \max(k_1, \dots, k_m)$ be the maximum degree of all monomials occurring in $\psi$.
  We add layers $l_2, \dots, l_{2+\hat{k}-1}$ so that they perform the same kind of computation as
  in Lemma~\ref{sec:polyconstr;lem:globhomconstr1}. This yields
  that $\bs{x}^{(2+(k_i-1))}_v$ contains $y_v=\frac{1}{|V|^{k_i}} \prod_{j=1}^{k_i} x_{i_j}$
  for all $i \leq m$ and $v \in V$.
  We utilise that there is a marked node $u$ to align the different denominators introduced while evaluating monomials.
  Specifically, in layer $l_{2+k_i-1}$, we add a gadget to $\comb_{2+k_i-1}$
  that computes $z_v = \relu(y_v - (1 - (\bs{x}^{(2+k_i-2)}_v)_\textit{mark}))$,
  where $(\bs{x}^{(2+k_i-2)}_v)_\textit{mark}$ denotes the dimension identifying the marked node $u$.
  We remark that it is no longer guaranteed to be dimension $p$ due to pontential shifts in the dimensionality of the states.
  Note that it is guaranteed that $y_v \in [0,1]$ for all $v \in V$ and, thus, $z_v = y_v$ if $u=v$ and $z_v=0$
  otherwise. Now, using mean global aggregation in layer $l_{2+k_i}$ over the dimension corresponding to
  $z_v$ computes $\frac{1}{|V|^{k_i+1}}\prod_{j=1}^{k_i} x_{i_j}$. By applying this construction repeatedly,
  we enable $M_\varphi$ to align all denominators to $|V|^{\hat{k}}$. 
  
  Similarly, we exploit the presence of a marked
  node $u$ and add the gadget $(\bs{x}^{(2)}_v)_b = \relu(|b| \cdot (\bs{x}^{(1)}_v)_\mathit{mark})$ to $\comb_2$.
  Note, that in the final evaluation of the corresponding Peano term, we multiply this value by $\operatorname{sgn}(b)$.
  Then, using global aggregation $\hat{k}$
  times over the dimension $(\bs{x}^{(1+i)}_v)_b$ computes the value $\frac{b}{|V|^{\hat{k}}}$. Note that this involves 
  using gadget $\relu((\bs{x}^{(1+i)}_v)_b  - \relu(b-b(\bs{x}^{(1+i)}_v)_\mathit{mark}))$ to store a nonzero value if and only
  if $v=u$.
  Given that we evaluated all monomials and constants that way, we refer to Lemma~\ref{app:omitted;lem:peanoformulas} 
  (with $r_1 = r_2 = \frac{1}{|V|^{\hat{k}}}$) stating how to evaluate $\psi$ and thereby 
  $\chi_i$ in $\comb_{2+\hat{k}-1}$ such that $\bs{x}^{(2+\hat{k}-1)}_v = \frac{1}{|V|^{\hat{k}}}$ 
  if $(G,v) \models \chi_i$ and $\bs{x}^{(2+\hat{k}-1)}_v = 0$ if $(G,v) \not\models \chi_i$.

  Finally, the remaining subformulas $\chi_i$, which are arbitrary Boolean formulas from $\PML_1$, are handled in $l_{2+\mathit{deg}(\varphi)-1}$ as follows.
  First, we adjust the value of $\chi_j$ with $j \leq h$ using gadget $\relu(\frac{1}{|V|^{\hat{k}}}-\relu(1-x_{\chi_j}))$, where $x_{\chi_j}$ represents 
  the previous evaluation of $\chi_j$. This  
  ensures that the semantics of all $\chi_j$ with $j < i$ are represented by $0$ and $\frac{1}{|V|^{\hat{k}}}$.
  Second, we add gadgets $\neg \chi_1 = \relu(\frac{1}{|V|^{\hat{k}}} - x_{\chi_1})$ and 
  $\chi_1 \land \chi_2 = \relu(x_{\chi_1} + x_{\chi_2} - \frac{1}{|V|^{\hat{k}}})$ and combine them as determined 
  by the structure of $\chi_i$. Note that the value $\frac{1}{|V|^{\hat{k}}}$ can be computed by repeated global mean aggregations, 
  analogously to how we adjusted the bias value $b$.
\end{proof}

\subsection*{Local Modalities Only}
While shallow PML formulas involving only global modalities are capable of expressing
interesting global properties, they cannot convey information about local neighbourhoods.
Thus, we turn our attention to $\PML^E_1$ next. 
\begin{example}
  Consider the formula $\varphi = \langle E_\text{in}, E_\text{out} \rangle_\psi(p_b, p_g \land \neg p_r)$,
  where $\psi(x_b, x_{g \land \neg r}) = x_b^2 - x_{g \land \neg r}^3 \leq 1$, and interpret $p_b$ as the colour \emph{blue}, 
  $p_g$ as the colour \emph{green} and $p_r$ as the colour \emph{red}. The formula $\varphi$ is satisfied by
  all pointed graphs $(G,v)$ where the squared number of ingoing neighbours of colour blue of $v$
  is at most as large as the cube of the number of outgoing neighbours of colour green but not red plus one of $v$.
\end{example}

\newcommand{\ind}[1]{\mathbf{1}_{#1}}
In contrast to the settings of 
Lemma~\ref{sec:polyconstr;lem:globhomconstr1} and Theorem~\ref{sec:polyconstr;th:globconstr1}, capturing formulas from $\PML^E_1$ with mean MPNN
obviously requires the use of local aggregation. However, 
local mean aggregation generally does not exhibit uniformity since each node's aggregated value depends
on the size of its neighbourhood. We identify that assuming the following conditions enables
MPNN from $\meanmpnn$ to recognise said formulas: (a) the graphs $(G,v)$ are regular, (b) $v$ is marked,
(c) $v$ has a self-loop.
\begin{theorem}
  \label{sec:polyconstr;th:localconstraintsmean}
    Let $\varphi \in \PML_1^{E}$.
    There is $M_\varphi \in \meanmpnn$ with $\mathcal{O}(\mathit{deg}(\varphi))$ layers 
    that recognises $\varphi$ over all regular pointed graphs in $(G,v) \in (\mathcal{G})^p_\bullet$ 
    where $v$ is strongly marked with certainty $\frac{1}{|V|^{\mathcal{O}(\mathit{deg}(\varphi))}}$.
\end{theorem}
\begin{proof}
  We enumerate the subformulas of $\varphi$ as a sequence $\chi_1, \dotsc, \chi_n$ such that if $\chi_j$ is a subformula
  of $\chi_i$, then $j \leq i$ holds, there is some $h \leq n$ such 
  that $\mathit{md}(\chi_j) = 0$ if and only if $j \leq h$, and $\chi_n = \varphi$. 
  Let $(G=(V,E,L),v) \in (\mathcal{G})^p_\bullet$ be a regular pointed graph where $v$ is marked and $v \in \neigh(v)$.
  Since $G$ is regular, we denote by $n_\text{in}$ the size of all ingoing and by $n_\text{out}$ the size of all outgoing neighbourhoods.
  We construct $M_\varphi$ as follows. Layer $l_1$ is constructed as in Lemma~\ref{app:omitted;lem:boolenformulas} in order to 
  evaluate all $\chi_i$ with $i \leq h$.

  Consider $\chi_i = \langle \pi_1, \dotsc, \pi_m \rangle_{\psi(x_1, \dotsc, x_m)}(\varphi_1, \dotsc, \varphi_m)$.
  By assumption, we know that $\mathit{md}(\chi_i) = 1$ and all $\pi_i \in \{E_\text{in}, E_\text{out}\}$, implying 
  that $M_\varphi$ has to evaluate $\psi$ 
  at node $v \in V$ over values $x_{i_j}$ equal to the number of nodes $u \in \neighin(v)$ or $u \in \neighout(v)$, depending 
  on the form of $\pi_{i_j}$, with $(G,u) \models \varphi_{i_j}$. 
  We add layers $l_{2}, \dotsc, l_{2k_i}$ to $M_\varphi$ to evaluate monomials $\prod_{j=1}^{k_i} x_{i_j}$ of $\chi_i$ in the following way: 
  W.l.o.g.\ assume that $\pi_{i_1} = E_\text{in}$. In layer $l_2$ we first compute $y_u = \frac{1}{n_\text{in}}x_{i_1}$ 
  for each node $u \in V$ using $\locagg_\text{in}$, 
  similarly to our constructions in Lemma~\ref{sec:polyconstr;lem:globhomconstr1} or Theorem~\ref{sec:polyconstr;th:globconstr1}.
  Then, we add the gadget $z_u=\relu(y_u-(1-(\bs{x}^{(1)}_u)_\textit{mark}))$ to $\comb_2$ where $\textit{mark}$
  corresponds to the state dimension that marks the focus $v$ (not necessarily equal to $p$ anymore due to potential shifts 
  in layer $l_1$). 
  Then, $z_u \neq 0$ at node $u$ if and only if $u = v$. W.l.o.g.\ assume that $\pi_{i_2} = E_\text{out}$.
  In layer $l_3$, we use $\locagg_\text{in}$ and add to $\comb_3$ a gadget
  which in combination computes
  $y_u' = \ind{G, u \models \varphi_{i_2}} \cdot \frac{1}{n_\text{in}} \sum_{w \in \neigh(u)} z_w$ for all $u \in V$.
  The respective gadget is identical to the one used in previous constructions to perform multiplication. Now, $y_u'$ is only non-zero for $u \in \neigh(v)$,
  yielding $y_u' = \ind{G, u \models \varphi_{i_2}} \cdot \frac{1}{n_\text{in}} y_v$.
  Next, in layer $l_4$ we use $\locagg_\text{out}$, and compute 
  $\frac{1}{n_\text{in}(n_\text{in}n_\text{out})} \prod_{j=1}^2 x_{i_j}$ at node $v$.
  Using this construction 
  repeatedly results in the value $\frac{1}{n_\text{in}(n_\text{in}n_\text{out})^{k_i-1}} \prod_{j=1}^{k_i} x_{i_j}$ 
  in layer $l_{2k_i}$. 
  
  To align necessarily introduced denominators for all monomial evaluations,
  we perform repeated ingoing or outgoing local aggregations until each denominator is 
  $\frac{1}{(n_\text{in}n_\text{out})^{\mathit{deg(\chi_i)}}}$. Note that this requires 
  adding $\relu(x - \relu(1- x_\text{mark}))$ gadgets to ensure that non-zero values occur only at $v$. 
  Similarly, we adjust layers $l_{2}, \dots, l_{2\mathit{deg}{\varphi}+1}$ to compute 
  $\frac{1}{(n_\text{in}n_\text{out})^{\mathit{deg}(\varphi)}}b$ 
  for all constant parts $b$ ocurring in $\psi$. From here onward, we proceed as stated in Lemma~\ref{app:omitted;lem:peanoformulas} 
  to evaluate $\chi_i$ using $r_1 = \frac{1}{(n_\text{in}n_\text{out})^{\mathit{deg}(\varphi)}}$ and $r_2 = \frac{1}{|V|^{2\mathit{deg}(\varphi)}}$, which we can compute using repeated global aggregation in
  in layers $l_1$ to $l_{2\mathit{deg}(\varphi)}$. 

  The remaining construction of $M_\varphi$ is exactly as described in the proof of Theorem~\ref{sec:polyconstr;th:globconstr1}.
\end{proof}
We remark that if we restrict to homogeneous constraints, assumption (c), meaning that the focus $v$ has a self-loop, 
could be dropped.

The limitation to regular graphs may restrict the applicability in most scenarios.
However, allowing for the use of sum or max aggregation functions alongside mean aggregations
enables us to show that such MPNN can recognise $\PML_1^E$ over all graphs that satisfy
properties (b) and (c) as previously discussed.
\begin{theorem}
  \label{sec:polyconstr;th:localconstraintssummeanmax}
    Let $\varphi \in \PML_1^{E}$ and $A = \{\text{mean}, \text{sum}\}$ or $A = \{\text{mean}, \text{max}\}$.
    There is $M_\varphi \in \mathcal{M}_A$ with $\mathcal{O}(\mathit{deg}(\varphi))$ layers 
    that recognises $\varphi$ over all graphs $(G,v) \in \mathcal{G}^p_\bullet$ where $v$ is strongly marked
    with certainty $\frac{1}{|V|^{\mathcal{O}(\mathit{deg}(\varphi))}}$.
\end{theorem}
\begin{proof}
    The construction of $M_\varphi$ works analogous to Theorem~\ref{sec:polyconstr;th:localconstraintsmean}.
    However, while evaluating monomials $\prod_{j=1}^{k_i} x_{i_j}$, we alternate between realisations of
    $\locagg_\text{out}$ and $\locagg_\text{in}$ either by mean or by sum (or max) aggregation.
    
    As done in the proof of Theorem~\ref{sec:polyconstr;th:localconstraintsmean}, we employ local mean aggregation, w.l.o.g\ assume its $\locagg_\text{in}$, 
    and a specific gadget in $l_2$ to compute $z_u = \frac{1}{|\neighin(v)|}x_{i_1}$ if $u = v$, and $z_u = 0$ otherwise.
    In layer $l_3$, instead of mean aggregation, we use outgoing local sum aggregation (or max aggregation) to compute
    $y_u' = \ind{G, u \models \varphi_{i_2}} \cdot \sum_{w \in \neigh(u)} z_w$ or 
    $y_u' = \ind{G, u \models \varphi_{i_2}} \cdot \max_{w \in \neigh(u)} z_w$ 
    for all $u \in V$. Subsequently, in $l_4$, we revert to local mean aggregation and continue
    as before. This construction introduces a denominator of the form $\frac{1}{n_{i_1} \cdot \, \dotsb \, \cdot n_{i_{\mathit{deg}(\varphi)}}}$
    where $n_{i_j} \in \{|\neighin(v)|, |\neighout(v)|\}$ to each evaluated monomial. However, in order to align 
    these denominators across all monomials and constants leads to $\frac{1}{|\neighin(v)|^{\mathit{deg}(\varphi)} \cdot |\neighout(v)|^{\mathit{deg}(\varphi)}} \geq \frac{1}{|V|^{2\mathit{deg}(\varphi)}}$
    in the worst case. 
    
    Aside from these adjustments, the construction proceeds exactly as in Theorem~\ref{sec:polyconstr;th:localconstraintsmean}.
\end{proof}

\subsection*{Mixed Modalities}

In the settings of Theorem~\ref{sec:polyconstr;th:globconstr1} or Theorem~\ref{sec:polyconstr;th:localconstraintsmean} and 
\ref{sec:polyconstr;th:localconstraintssummeanmax}, we assumed that modal formulas use uniform modalities, in
the sense that they are either $\top$ or $E_\text{in}, E_\text{out}$, but not both simultaneously. Furthermore, we
did not yet consider the modality $\text{id}$. However, combinations of these modalities bring the full expressive
power of $\PML_1$.
\begin{example}
Let $\varphi = \langle E_\text{in}, \top, \text{id} \rangle(p_r \land p_b, p_b, p_g \lor p_y)$, where we
use $\lor$ as the usual abbreviation for logical disjunction, and $\psi(x_{r\land b}, x_b, x_{g\lor y}) =
(16 \leq x_{r\land b}^2 + x_bx_{g \lor y}) \land (x_{b}^3 + x_{r \land b}(1 - x_{g \lor y}) \leq 64)$. 
As in previous examples, we interpret $p_r$ as colour \emph{red}, $p_b$ as colour \emph{blue}, $p_g$ as
colour \emph{green}, and $p_y$ as colour \emph{yellow}. Informally, the formula $\varphi$ is satisfied by
  exactly those pointed graphs $(G,v)$ where 
  \begin{itemize}
  \item if $v$ is of colour green or yellow, we have that
  the square of the number of ingoing neighbours of $v$ which are red and blue plus the global number of blue nodes
  is at least $16$ ($16 \leq x_{r\land b}^2 + x_b\cdot 1$), and the cube of the number of blue nodes in $G$ is
  at most $64$ ($x_{b}^3 + x_{r \land b}\cdot 0\leq 64$); or
  \item if $v$ is neither of colour green nor yellow, we have that the square of the
  number of ingoing neighbours of $v$ which are red and blue is at least $16$ ($16 \leq x_{r\land b}^2 + x_b\cdot 0$),
  and the cube of the number of blue nodes in $G$ plus the number of ingoing neighbours of $v$ that are
  red and blue is at most $64$ ($x_{b}^3 + x_{r \land b}\cdot 1 \leq 64$).
  \end{itemize}
\end{example}

As it turns out, in order to capture all formulas from $\PML_1$, it suffices to combine
the constructions developed in previous sections.
\begin{theorem}
  \label{sec:polyconstr;th:mixedmodmean}
    Let $\varphi \in \PML_1$.
    There is $M_\varphi \in \meanmpnn$ with $\mathcal{O}(\mathit{deg}(\varphi))$ layers 
    that recognises $\varphi$ over all regular graphs in $(G,v) \in (\mathcal{G})^p_\bullet$ 
    where $v$ is strongly marked with certainty $\frac{1}{|V|^{\mathcal{O}(\deg(\varphi))}}$.
\end{theorem}
\begin{proof}
The proof follows the same line of arguments as used in the proofs of Theorem~\ref{sec:polyconstr;th:globconstr1} and
  Theorem~\ref{sec:polyconstr;th:localconstraintsmean}.

To allow $M_\varphi$ to evaluate monomials $\prod_{j=1}^{k_i} x_{i_j}$ for a given regular pointed graph
  $(G=(V,E,L), v)$, where $x_{i_j}$ corresponds to a modality $ \pi_{i_j} \in \{\text{id}, E_\text{in}, E_\text{out}, \top\}$, we adjust and combine previous constructions as follows. If $x_{i_j}$ corresponds to $\top$,
we evaluate it as described in Theorem~\ref{sec:polyconstr;th:globconstr1} using global mean aggregation,
thus introducing denominators of the form $\frac{1}{|V|^i}$. If
$x_{i_j}$ corresponds to $E_\text{in}$ or $E_\text{out}$, we evaluate it as described in
Theorem~\ref{sec:polyconstr;th:localconstraintsmean}, which introduces denominators of the form
$\frac{1}{(n_\text{in}n_\text{out})^i}$, where $n_\text{in}$ and $n_\text{out}$ are the sizes of ingoing and outgoing
neighbourhoods of the regular graph $G$. If $x_{i_j}$ corresponds to $\text{id}$, meaning aggregation
over the focus node $v$ only, we use the fact that $v$ is marked, and (i) adjust all dimensions $x$ of interest
by $\relu(x - \relu(1-x_\text{mark}))$, where $x_\text{mark}$ corresponds to the dimension marking $v$, storing the result in new dimensions;
(ii) use global or local aggregation. Note that depending on the choice this introduces
$\frac{1}{|V|^i}$, $\frac{1}{n_\text{in}^i}$, or $\frac{1}{n_\text{out}^i}$. Finally, we use repeated global, ingoing,
or outgoing aggregation to align all denominators introduced in the values representing monomials and constants.
We remark that this requires a linear amount of layers regarding $\mathit{deg}(\varphi)$, depending 
on distribution of different modalities across monomials.

Otherwise, the construction remains the same as in previous proofs.
\end{proof}

Allowing for either sum or max aggregations besides mean aggregations allows to drop the regularity
assumption. The proof requires the exact same adjustments as made in the step from 
Theorem~\ref{sec:polyconstr;th:localconstraintsmean} to Theorem~\ref{sec:polyconstr;th:localconstraintssummeanmax}.
\begin{theorem}
  \label{sec:polyconstr;th:mixedmodmeansummax}
  Let $\varphi \in \PML_1$ and $A = \{\text{mean}, \text{sum}\}$ or $A = \{\text{mean}, \text{max}\}$.
  There is $M_\varphi \in \mathcal{M}_A$ with $\mathcal{O}(\mathit{deg}(\varphi))$ layers 
  that recognises $\varphi$ over all graphs in $(G,v) \in (\mathcal{G})^p_\bullet$ 
  where $v$ is strongly marked with certainty $\frac{1}{|V|^{\mathcal{O}(\mathit{deg}(\varphi))}}$.
\end{theorem}

%% file: sections/nested_polyconstraints.tex
In this section, we consider fragments of $\PML$ of arbitrary modal depth. Contrary to the settings of Lemma~\ref{sec:polyconstr;lem:globhomconstr1} and
Theorem~\ref{sec:polyconstr;th:globconstr1}, focusing on fragments with nested modal formulas but 
using only $\top$ as a modality does not bring any new insights into the power of mean MPNN 
since nesting does not enhance the expressive capabilities of these fragments compared to $\PML_1^\top$, 
as demonstrated in the following.
%\begin{lemma}\label{lem:neg-pml}
%For any PML formulas $\varphi_1,\dots,\varphi_m$, Peano formula $\psi(x_1,\dots,x_m)$, and modalities $\pi_1,\dots,\pi_m$
%the following PML formulas are equivalent:
%\begin{itemize}
%\item $\neg(\langle \pi_1, \dotsc, \pi_m \rangle_{\psi(x_1, \dotsc, x_m)}(\varphi_1, \dotsc, \varphi_m))$
%\item $\langle \pi_1, \dotsc, \pi_m \rangle_{\neg\psi(x_1, \dotsc, x_m)}(\varphi_1, \dotsc, \varphi_m)$
%\end{itemize}
%\end{lemma}
%\begin{proof}
%Let $(G,v)$ be a pointed graph and 
%$n_i := |\{u \in \sem{\pi_i}^G_v \mid G,u \models \varphi_i\}|$ for all $i \le m$.
%Note that the $n_i$ are completely determined by $(G,v)$.
%By definition, 
%\begin{align*}
%G,v &\models \neg(\langle \pi_1, \dotsc, \pi_m \rangle_{\psi(x_1, \dotsc, x_m)}(\varphi_1, \dotsc, \varphi_m)) \text{ iff}\\
%G,v &\not\models \langle \pi_1, \dotsc, \pi_m \rangle_{\psi(x_1, \dotsc, x_m)}(\varphi_1, \dotsc, \varphi_m) \text{ iff}\\
%n_1,\dots,n_m &\not\models \psi(x_1,\dots,x_m) \text{ iff}\\
%n_1,\dots,n_m &\models \neg\psi(x_1,\dots,x_m) \text{ iff}\\
%G,v &\models \langle \pi_1, \dotsc, \pi_m \rangle_{\neg\psi(x_1, \dotsc, x_m)}(\varphi_1, \dotsc, \varphi_m).
%\end{align*}
%\end{proof}

\begin{proposition}\label{prop:global-pml-depth}
For all $\varphi \in \PML^\top$ there is $\varphi' \in \PML^\top_1$ with $\sem{\varphi} = \sem{\varphi'}$. 
\end{proposition}
\begin{proof}
Let $\varphi $ be a PML formula with only global modalities.
%We first use De Morgan's laws and Lemma~\ref{lem:neg-pml} to assume that all negations in $\varphi$ are moved inwards 
%such that negations only appear directly in front of propositions and inside Peano formulas.
For a set $T$ of modal formulas and a PML formula $\gamma$, we write
$\gamma^T$ for the formula where every strict modal subformula $\mu$ of $\gamma$ is replaced 
with $\top$ if $\mu \in T$ and with $\bot$ otherwise.  
Here, strict means that if $\gamma$ is a modal formula, then $\gamma$ itself is neither replaced with $\top$ nor $\bot$.
Let $S$ be the set of strict modal subformulas of $\varphi$.
Now $\varphi$ is equivalent to
\[\bigvee_{T \subseteq S} \big( \varphi^T \wedge \bigwedge_{\tau \in T} \tau^T \wedge \bigwedge_{\sigma \in S \setminus T} \neg \sigma^T \big)\]
which is a formula of modal depth at most 1.
%Note that since negations do not appear outside of modal subformulas,
%we only have to check that the modal formulas in $T$, whose satisfiability make $\varphi$ evaluate to true,
%are satisfied.
Here, the modal subformulas can be checked without nesting them since they only have global modalities,
i.e., all modal formulas are  evaluated over the same pointed graph.
\end{proof}

\begin{figure}
  \centering 
  \input{graphics/treelikes.tex}
  \caption{Let $\varphi$ be a formula with $\md(\varphi) = 2$ that only uses the modality $E_\text{out}$.
  The top two graphs are members of $(\mathcal{T}^\circlearrowright_{\varphi})^p_\bullet$, where the blue
  node indicates the focus. The bottom graphs are not included in $(\mathcal{T}^\circlearrowright_{\varphi})^p_\bullet$
  since multiple predecessors of the red node have the same trace set. For clarity reasons, self-loops are not depicted.}
  \label{sec:nested;fig:treelikeexamples}
\end{figure}
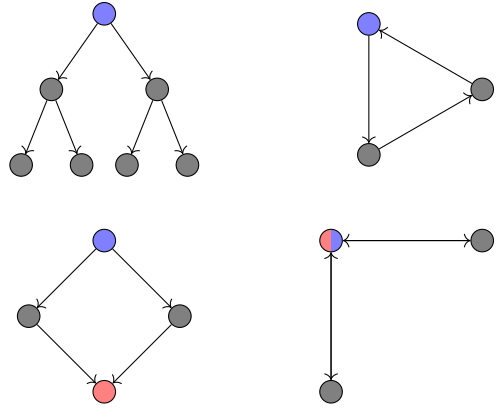

However, the situation is entirely different if we consider local modalities $E_\text{in}, E_\text{out}$ only. 
Let $\varphi \in \PML^{E}$ and $(G = (V,E,L),v)$ be a pointed graph. We denote by 
$\walks^{\leq i}_\varphi(u)$ the set of all traces $E_0, \dotsc, E_{j-1} \in \walks^j_\varphi$ with $j \leq i$ such that there is a 
walk from $v$ to $u$ respecting $E_0, \dotsc, E_{j-1}$.
We define the class $(\mathcal{T}^\circlearrowright_{\varphi})^p_\bullet$ of 
all pointed graphs $(G = (V,E,L),v)$ where $v$ is marked and for all $w \in V$ and
$1 \leq i \leq \mathit{md}(\varphi)$ if there is a walk $u_0 u_1 \dotsb u_{i-1} w$ with $u_0 = v$
that corresponds to a trace $E_0, \dotsc, E_{i-1} \in \walks^i_\varphi$
then we have $u_{i-1} \in \neigh(u_{i-1})$ and 
for all $w_1, w_2 \in \neigh_x(w)$ with $x = \text{in}$ if $E_{i-1} = E_\text{out}$ and  
$x = \text{out}$ if $E_{i-1} = E_\text{in}$ we have that if $\walks^{\leq i-1}_\varphi(w_1) = \walks^{\leq i-1}_\varphi(w_2)$
then $w_1 = w_2$. 

While the definition of $(\mathcal{T}^\circlearrowright_{\varphi})^p_\bullet$ is rather technical, assuming for instance that $\modalities{i}{\varphi} = \{E_\text{out}\}$ for all $i \leq \md(\varphi)$ implies that 
$(\mathcal{T}^\circlearrowright_{\varphi})^p_\bullet$ includes trees with self-loops where the focus is the root. 
In Figure~\ref{sec:nested;fig:treelikeexamples}, we depict some examples. We identify $(\mathcal{T}^\circlearrowright_{\varphi})^p_\bullet$ 
as a suitable candidate to frame the expressive capabilities of mean MPNN compared to $\PML^{E}$.
\begin{theorem}
    \label{sec:nested;th:meanmpnn}
    Let $\varphi \in \PML^{E}$.
    There is $M_\varphi \in \meanmpnn$ with $\mathcal{O}(\mathit{md}(\varphi)\textit{deg}(\varphi))$ layers
    that recognises $\varphi$ over all regular graphs in $(\mathcal{T}^\circlearrowright_{\varphi})^p_\bullet$ 
    with certainty $\frac{1}{|V|^{\mathcal{O}(\mathit{md}(\varphi)\textit{deg}(\varphi))}}$.
\end{theorem}
\begin{proof}
  Let $\chi_1, \dotsc, \chi_n$ be an enumeration of the subformulas of $\varphi$ such that if $\chi_j$ is a subformula
  of $\chi_i$, then $j \leq i$, there is some $h \leq n$ such 
  $\mathit{md}(\chi_i) = 0$ if and only if $i \leq h$ for all $i \leq n$, and we have $\chi_n = \varphi$.
  W.l.o.g.\ we assume that $p=0$. Let $(G=(V,E,L),v) \in (\mathcal{T}^\circlearrowright_{\varphi})^p_\bullet$.
  Similar to previous proofs, we describe how to construct $M_\varphi$. 
  Layer $l_1$ is constructed as described by Lemma~\ref{app:omitted;lem:boolenformulas} and used to evaluate all 
  $\chi_i$ with $i \leq h$.

  Next, we use layers $l_1, \ldots, l_{\mathit{md}(\varphi)-1}$ to ``propagate'' the mark of the focus
  throughout the graph along traces of $\walks^{i}_\varphi$ for $i \leq \md(\varphi)-1$. In layer $l_i$ we use $\locagg_\text{in}$ if
  there is trace $T=E_0, \dotsc, E_{i-1} \in \walks^i_\varphi$ with $E_{i-1} = E_\text{out}$ and 
  $\locagg_\text{out}$ if
  there is trace $T=E_0, \dotsc, E_{i-1} \in \walks^i_\varphi$ with $E_{i-1} = E_\text{in}$
  to aggregate for all $u \in V$ a new flag $y_T$ stored in some dimension associated with $T$.
  Simultaneously, we use $\globagg$ in each layer $l_i$ to compute $z_i = \frac{1}{|V|^i}$ through repeated aggregation over the marking dimension 
  as done in previous constructions. 
  Finally, in $\comb_i$, we compute $\min(z_i, y_T)$ for each $T \in \walks^i_\varphi$.
  This ensures that we stored $\frac{1}{|V|^i}$ in a dimension associated with $T$ 
  if there exists a walk from $v$ to $u$ respecting $T \in \walks_\varphi^i$, and its $0$ otherwise.

  Next, in layer $l_{\mathit{md}(\varphi)}$, we employ the value $\frac{1}{|V|^{\mathit{md}(\varphi)-1}}$,
  obtained through repeated global aggregations in layers $l_1, \ldots, l_{\mathit{md}(\varphi)-1}$, to compute
  $(\bs{x}^{(\mathit{md}(\varphi))}_u)_{i} = \relu(\frac{1}{|V|^{\mathit{md}(\varphi)-1}} - \relu(1 - (\bs{x}^{(\mathit{md}(\varphi)-1)}_u)_{i}))$
  for all $u \in V$ and $i\leq h$. Informally, this means a shift of the representative values for all so far evaluated
  formulas (which are exclusively Boolean) with $0 \mapsto 0$ and $1 \mapsto \frac{1}{|V|^{\mathit{md}(\varphi)-1}}$.

  Now, consider subformula $\chi_i = \langle \pi_{i,1}, \dots, \pi_{i,m} \rangle_{\psi_i(x_{i,1}, \ldots, x_{i,m})}(\varphi_{i,1}, \ldots, \varphi_{i,m})$
  with $\mathit{md}(\chi_i) = 1$. First, we argue how to evaluate monomials $\prod_{j=1}^{k_i} x_{i_j}$ in $\psi_i$.
  W.l.o.g assume that the modal depth of $\chi_i$ implies that $\chi_i$ must be evaluated at nodes $u$ such that there exists a walk
  starting from $v$ to $u$ respecting some $T \in \walks^{\mathit{md}(\varphi)-1}_\varphi$. We remark that the structure of $\varphi$ exactly characterises 
  at which distance from $v$ some $\chi_i$ must be evaluated. 
  We fix such a node $u$ in the following.
  We utilise layers $l_{\mathit{md}(\varphi)+1}, \dotsc, l_{\mathit{md}(\varphi)+2\mathit{deg}(\varphi)-1}$ to evaluate $\prod_{j=1}^{k_i} x_{i_j}$ as done in
  Theorem~\ref{sec:polyconstr;th:localconstraintsmean}, but with the following adjustments. W.l.o.g.\ assume 
  that $x_{i_1}$ corresponds to modality $E_\text{out}$ and $x_{i_2}$ corresponds to $E_\text{in}$. 
  We use $\locagg_\text{out}$ in layer $l_{\mathit{md}(\varphi)+1}$ as before. 
  Now, in $\comb_{\mathit{md}(\varphi)+1}$ we add gadget 
  $\relu(x - \sum_{T \in I} (\frac{1}{|V|^{|T|}} - y_T) - \sum_{T \in \walks^{\leq \md(\varphi)-1}_\varphi\setminus I} y_T)$
  for all $I \subseteq \walks^{\leq \md(\varphi)-1}_\varphi$ where $y_T$ is the flag we previously computed 
  with $y_T=\frac{1}{|V|^{|T|}}$ if there is a walk from 
  $v$ to $u$ respecting $T$. 
  This ensures that we preserve a non-zero value corresponding to the partial evaluation of 
  $\prod_{j=1}^{k_i} x_{i_j}$ at $u$ in a dimension identified by $\walks^{\leq \md(\varphi)-1}_\varphi(u)$. 
  Note that in order for this gadget to work correctly we require the shift in values (0 to 0 and 1 to $\frac{1}{|V|^{\mathit{md}(\varphi)-1}}$) 
  which is performed by layer $l_{\md(\varphi)}$.
  Due to the definition of $(\mathcal{T}^\circlearrowright_{\varphi})^p_\bullet$ 
  it is ensured that all nodes $w \in \neighin(u)$ have at most one $u' \in \neighout(w)$ with a specific trace set $\walks^{\leq \md(\varphi)-1}_\varphi(u')$.
  Thus, the evaluation of $\prod_{j=1}^{k_i} x_{i_j}$ proceeds as before, but separately for each dimension corresponding to some 
  $\walks^{\leq \md(\varphi)-1}_\varphi(u')$ with $u' \in V$.
  In the end, we utilise that 
  $u$ has a self-loop, again as given by the definition of $(\mathcal{T}^\circlearrowright_{\varphi})^p_\bullet$, to align the denominators 
  for all monomials by repeated local aggregations along this self-loop. Note that this only works due to $G$ being regular, similar to 
  Theorem~\ref{sec:polyconstr;th:localconstraintsmean}. Given that 
  we evaluated all monomials in $\psi_i$ this way and normalised constant parts of $\psi_i$ in the same way, we 
  evaluate $\chi_i$ as described in Lemma~\ref{app:omitted;lem:peanoformulas} with $r_2 = \frac{1}{|V|^{\mathit{md}(\varphi)-1+2\mathit{deg}(\varphi)}}$, computed 
  using repeated global aggregations in previous layers.

  Now, for modal formulas $\chi_i$ with $\mathit{md}(\chi_i) > 1$ we proceed analogously, but we adjust the values representing 
  all so far evaluated formulas to $\frac{1}{|V|^{\mathit{md}(\varphi)+((\mathit{md}(\chi_i)-1)2\mathit{deg}(\varphi))}}$ and $0$, assuming that $\chi_i$ must be evaluated 
  at nodes $u$ with distance $\md(\varphi)-\md(\chi_i)$ from $v$. For all $\chi_j$ with $j \leq h$ 
  this works as described before in preparation of handling $\mathit{md}(\chi_i)=1$ subformulas. For $\chi_j$ which are already evaluated modal formulas we take the minimum of 
  the currently representing value and $\frac{1}{|V|^{\mathit{md}(\varphi)+((\mathit{md}(\chi_i)-1)2\mathit{deg}(\varphi))}}$.
  Then, after resolving all modal formulas 
  with $\md(\chi_i) = \md(\varphi)$ we evaluate the remaining subformulas as in Theorem~\ref{sec:polyconstr;th:localconstraintsmean}, but 
  with $\frac{1}{|V|^{\mathit{md}(\varphi)(2\mathit{deg}(\varphi))}}$. Note that this, again, requires shifting representative values.  
\end{proof}

Analogous to Theorem~\ref{sec:polyconstr;th:localconstraintssummeanmax}, 
we get that also allowing for sum or max as aggregation besides taking the mean in turn 
allows us to drop the regularity assumption.
\begin{theorem}
  \label{sec:nested;th:mpnnmeansummax_and_mixed}
  Let $\varphi \in \PML^E$, and $A \in \{\{\text{mean}, \text{sum}\}, \{\text{mean}, \text{max}\}\}$. 
  There is $M_\varphi \in \mathcal{M}_A$ that recognises $\varphi$ over all graphs in $(\mathcal{T}^\circlearrowright_{\varphi})^p_\bullet$ 
  with certainty $\frac{1}{|V|^{\mathcal{O}(\mathit{md}(\varphi)\textit{deg}(\varphi))}}$.
\end{theorem}

%% file: graphics/treelikes.tex
\begin{tikzpicture}[
  dot/.style={draw,circle,fill=black!50,inner sep=3pt},
  highlightR/.style={draw,circle,fill=blue!50,inner sep=3pt},
  highlightB/.style={draw,circle,fill=red!50,inner sep=3pt},
    halfRB/.style={
    draw,
    circle,
    inner sep=3pt,
    path picture={
      % left half red
      \fill[red!50]
        (path picture bounding box.south west)
        rectangle
        (path picture bounding box.north);
      % right half blue
      \fill[blue!50]
        (path picture bounding box.south)
        rectangle
        (path picture bounding box.north east);
    }
  }
]

% --- tree ---

% highlighted root
\node[highlightR] (r)  at (0,0) {};

\node[dot] (a)  at (-.7,-1) {};
\node[dot] (b)  at (.7,-1) {};

\node[dot] (a1) at (-1.1,-2) {};
\node[dot] (a2) at (-.3,-2) {};
\node[dot] (b1) at (.3,-2) {};
\node[dot] (b2) at (1.1,-2) {};

\draw[->] (r) -- (a);
\draw[->] (r) -- (b);
\draw[->] (a) -- (a1);
\draw[->] (a) -- (a2);
\draw[->] (b) -- (b1);
\draw[->] (b) -- (b2);

% --- 3-cycle ---

\def\cx{4}
\def\cy{-1}
\def\r{1}

% nodes (c1 highlighted as before)
\foreach \i in {1,...,3}{
  \ifnum\i=1
    \node[highlightR] (c\i) at ({\cx+\r*cos(120*\i)},{\cy+\r*sin(120*\i)}) {};
  \else
    \node[dot] (c\i) at ({\cx+\r*cos(120*\i)},{\cy+\r*sin(120*\i)}) {};
  \fi
}

% cycle edges
\foreach \i [evaluate=\i as \j using {int(mod(\i,3)+1)}] in {1,...,3}{
  \draw[->] (c\i) -- (c\j);
}

% DIAMOND
% nodes
\node[highlightR] (v) at (0,-3) {};
\node[dot] (l) at (-1,-4) {};
\node[dot] (r) at (1,-4) {};
\node[highlightB] (b) at (0,-5) {};

% edges
\draw[->] (v) -- (l);
\draw[->] (v) -- (r);
\draw[->] (l) -- (b);
\draw[->] (r) -- (b);

% RECTANGLE 
\node[halfRB] (a) at (3,-3) {};
\node[dot] (b) at (5,-3) {};
\node[dot] (d) at (3,-5) {};

% edges (both directions)
\draw[->] (a) -- (b);
\draw[->] (b) -- (a);

\draw[->] (d) -- (a);
\draw[->] (a) -- (d);
\end{tikzpicture}

%% file: sections/outlook.tex
We investigated the capabilities of MPNN, primarily using mean aggregation, in recognising 
polynomial count properties in graphs. We demonstrated that, without further assumptions,
mean MPNN that are capable of capturing
homogeneous polynomial constraints over global properties
(Lemma~\ref{sec:polyconstr;lem:globhomconstr1}).
Subsequently, we proved that a simple assumption about pointed graphs, namely
that the focus is marked, allows mean MPNN to recognise arbitrary
polynomial constraints over regular pointed graphs, counting globally or locally without nesting
(Theorem~\ref{sec:polyconstr;th:globconstr1} and Theorem~\ref{sec:polyconstr;th:localconstraintsmean}).
Interestingly, we found that by allowing additional sum or max aggregations, the
regularity assumption can be omitted
(Theorem~\ref{sec:polyconstr;th:mixedmodmeansummax}).
Considering nested modalities, we demonstrated that for graphs with a somewhat tree-like structure,
the mean plus sum or max MPNN can capture all polynomial constraints
(Theorem~\ref{sec:nested;th:mpnnmeansummax_and_mixed}).
Again, when considering MPNN using only mean aggregation, we required a restriction
to regular, tree-like graphs
(Theorem~\ref{sec:nested;th:meanmpnn}).

Building on our findings, future work should explore to what extent we can generalise
PML along the modality dimension. For example, whether allowing arbitrary
Boolean combinations of the atomic modalities $\text{id}$, $E_\text{in}$, and $E_\text{out}$
leads to similar findings. Additionally, it can be envisaged that MPNN are also capable
of comparing counts of neighbourhoods at different distances to the focus point. However,
this is not incorporated into the logic PML and, thus, requires additional work.

%% file: sections/proofs.tex
In the following technical lemmas and statements, we always assume that 
$\comb$ functions are realised by FNN with ReLU ($\relu(x) = \max(0,x)$) 
activations.

We denote by PL the fragment of \PML{} formulas $\varphi$ such that 
there is no modal formula $\varphi' = \langle \pi_1, \dotsc, \pi_m \rangle_{\psi(x_1, \dotsc, x_m)}(\varphi_1, \dotsc, \varphi_m)$ with 
$\varphi' \in \mathit{sub}(\varphi)$. In other words, PL are effectively propositional formulas.
\begin{lemma}
  \label{app:omitted;lem:boolenformulas}
  Let $\varphi_1, \dotsc, \varphi_n \in \text{PL}$ over the finite set of propositions $P$. There is a 
  combination function $\comb: \{0,1\}^{4|P|} \to \{0,1\}^{n} \times \{0\}^m$ for some $m \in \nats$, 
  such that for all graphs $G = (V,E,L)$ with $|P|$ colours, nodes $v \in V$, and $i \in \{1, \dotsc, n\}$
  we have  
  \begin{align*}
    (\comb(\bs{x}^{(0)}_v,\bs{y}_{1}, \bs{y}_{2},\bs{y}_{3}))_i= 1 && \text{iff} && G,v \models \varphi_i
  \end{align*}
  where $\bs{y}_{1}$, $\bs{y}_{2}$, and $\bs{y}_{3}$ are the results coming from any form of local and global aggregations.
\end{lemma}
\begin{proof}
    Due to the fact that all formulas $\varphi_1, \dotsc, \varphi_n \in \text{PL}$,
    implying that we do not need to use any information coming from local or global aggregation to evaluate these.
    Therefore, we construct the FNN $N$ representing $\comb$ such that it simply ignores the inputs 
    $\bs{y}_{1}$, $\bs{y}_{2}$, and $\bs{y}_{3}$. This is done by weighting these inputs with $0$ in the 
    first layer of $N$. 

    Let $G=(V,E,L)$ be a graph with $|P|$ colours and $v \in V$.
    We proceed the construction of $N$ as follows. If 
    $\chi_i = p_j$, we add the gadget $\relu((\bs{x}^{(0)}_v)_j)$ to $N$. In other
    cases, we use $\relu(1-x_j)$ if $\chi_i = \neg \chi_j$,
    and $\relu(x_{j_1} + x_{j_2} - 1)$ if $\chi_i = \chi_{j_1} \land \chi_{j_2}$.
    Here, the values $x_j$, $x_{j_1}$ and $x_{j_2}$ may
    not necessarily correspond to dimensions of $\bs{x}^{(0)}_v$,
    but can also correspond to the outputs of previously applied gadgets.    
\end{proof}

Let $\psi(x_1, \dotsc, x_n)$ be a Peano formula. 
We denote by $\textit{Mon}_\psi$ the set of all monomials $\prod_{j=1}^{k_i} x_{i_j}$ occurring in terms of $\psi$ and 
$\mathit{Const}_\psi$ the set of all constants $b$ occurring in terms of $\psi$.
\begin{lemma}
  \label{app:omitted;lem:peanoformulas}
  Let $\psi(x_1, \dotsc, x_n)$ be a Peano formula with 
  $\textit{Mon}_\psi = \{M_1, \dotsc, M_m\}$ and $\textit{Const}_\psi = \{b_1, \dotsc, b_k\}$ 
  , and $(c_1, \dotsc, c_n) \in \nats^n$ an assignment for $\psi$. For all $r_1,r_2 \in \rats$ with $r_1 \geq r_2 \geq 0$ 
  there is combination function $\comb$ such that
  \begin{align*}
    (\comb(\bs{x},\bs{y}_1,\bs{y}_2,\bs{y}_3))_i = r_2 && \text{iff} && (c_1, \dotsc, c_n) \models \psi,
  \end{align*}
  where $\bs{x}$ contains the values $r_1\cdot M_i(c_1, \dotsc, c_n)$, $r_1\cdot b_j$, $r_2$, for all $i\leq m$ and 
  $j \leq k$, and 
  $\bs{y}_{1}$, $\bs{y}_{2}$, and $\bs{y}_{3}$ are the results coming from any form of local and global aggregations.
\end{lemma}
\begin{proof}
  Similar to Lemma~\ref{sec:polyconstr;lem:globhomconstr1}, whether a Peano term $t = \sum_{i=1}^K a_i \cdot \prod_{j=1}^{k_i} x_{i_j} \leq b$
  is satisfied by $(c_1, \dotsc, c_n)$ is checked using the gadget 
  $x_{t} = \relu((r_1\sum_{i=1}^K a_i \cdot \prod_{j=1}^{k_i} x_{i_j}) - r_1b)$ in the sense that $x_t = 0$ if 
  $(c_1, \dotsc, c_n) \models t$ and $x_t \geq r_1$ if not. Next, we use gadget $y_t=\relu(x_t - \relu(x_t-r_2))$ to compute $y_t = \min(x_t,r_2)$. 
  Finally, we use $z_t=\relu(r_2 - y_t)$ to get that $z_t = 0$ if $(c_1, \dotsc, c_n) \not\models t$ and $z_t = r_2$ if $(c_1, \dotsc, c_n) \models t$.
  
  To fully evaluate $\psi$, which potentially is a Boolean combination of Peano terms, we 
  we handle Boolean combinations using gadgets $\neg \psi_1 = \relu(r_2 - x_{\psi'})$ and 
  $\psi_1 \land \psi_2 = \relu(x_{\psi_1} + x_{\psi_2} - r_2)$ and combine them as determined 
  by the structure of $\psi$.
\end{proof}